\documentclass{article}
\usepackage{ijcai25}

\usepackage{amsmath}
\usepackage{amsthm} 
\newtheoremstyle{thmstyleone} 
  {3pt}          
  {3pt}          
  {\itshape}     
  {}             
  {\bfseries}    
  {.}            
  {.5em}         
  {}             
\newtheorem{theorem}{Theorem}

\newtheorem{lemma}[theorem]{Lemma}

\newtheorem{assump}[theorem]{Assumption}
\newtheorem{proposition}[theorem]{Proposition}

\usepackage{times}
\usepackage{soul}
\usepackage{url}
\usepackage[utf8]{inputenc}
\usepackage[small]{caption}
\usepackage{booktabs}
\usepackage[switch]{lineno}
\usepackage{natbib}
\usepackage[psdextra]{hyperref}

\usepackage{tikz}
\usepackage{pgfplots}
\usepackage{subcaption} 
\pgfplotsset{compat=1.18}

\usepackage{amsmath,amssymb,amsfonts}

\usepackage{algorithm}
\usepackage{algpseudocode}
\usepackage{booktabs}
\usepackage{multirow}
\usepackage{placeins}
\algnewcommand{\Phase}[1]{%
   \State \textbf{#1}%
}
\algnewcommand{\Var}[1]{\texttt{#1}}

\usepackage{graphicx}
\usepackage{textcomp}
\usepackage{xcolor}
\usepackage{multirow}
\usepackage{mathrsfs}
\usepackage{siunitx}  
\usepackage{graphicx}
\usepackage{booktabs}
\usepackage{placeins}

\providecommand{\norm}[1]{\ensuremath{\left\lVert #1 \right\rVert_2}}
\providecommand{\fro}[1]{\ensuremath{\left\lVert #1 \right\rVert_F}}

\DeclareMathOperator{\tr}{tr}

\providecommand{\grad}{\nabla}

\def\BibTeX{{\rm B\kern-.05em{\sc i\kern-.025em b}\kern-.08em
    T\kern-.1667em\lower.7ex\hbox{E}\kern-.125emX}}
    
\hypersetup{colorlinks = true, 
	    linkcolor = black, 
	    urlcolor = blue,
           citecolor = blue} 

\begin{document}

\title{Adversarial LassoNet: Robust Feature Selection via Stability-Driven Sparse Learning}

\author{
Zhen Huang\textsuperscript{\dag},
Peicheng Xu\textsuperscript{\dag},
Junbiao Pang\textsuperscript{*},
Yulong Zheng,
Yuan Chen\\
\affiliations
Beijing University of Technology, The First Affiliated Hospital of Zhejiang University School of Medicine, Peking University First Hospital, \\
\textsuperscript{\dag}Equal contribution. 
\textsuperscript{*}Corresponding author.\\
\emails
huangzhen@emails.bjut.edu.cn,
pcxu45635@gmail.com,
junbiao\_pang@bjut.edu.cn,
1507138@zju.edu.cn,
softsnake@bjmu.edu.cn
}


\maketitle

\begin{abstract}

Sparse feature selection is critical for high-dimensional machine learning, yet traditional $\ell_1$-regularized methods are often brittle under observational noise and spurious correlations, leading to unstable feature supports and degraded generalization. Although adversarial training has been widely used to improve model robustness, its interaction with hierarchical sparse feature selection remains underexplored. In this work, we propose Adversarial LassoNet (AdLNet), a stability-driven sparse feature selection framework that integrates input-space adversarial perturbations with the hierarchical sparsity mechanism of LassoNet. We derive a tractable first-order adversarial approximation under local smoothness assumptions and provide an NTK-inspired spectral analysis to characterize how perturbation-driven training can reduce gradient concentration. Experiments on high-dimensional SERS data, six public benchmark datasets, and ColoredMNIST show that AdLNet maintains competitive sparse-selection performance while improving out-of-distribution robustness by 4.4\% and feature support reproducibility by 6.3\% under nearly matched support sparsity on ColoredMNIST. On the high-dimensional lung cancer screening dataset, AdLNet achieves a 5.3\% test accuracy gain and a 6.0\% AUC improvement over vanilla LassoNet. Code and dataset are available at \url{https://github.com/719573/Adversarial-LassoNet}.
\end{abstract}


\section{Introduction}

Feature selection is a fundamental task in machine learning, particularly in high-dimensional settings where redundant or noisy variables can substantially degrade predictive performance and generalization~\cite{guyon2003feature,lemhadri2021lassonet}. Selecting a compact feature set is especially important when the number of input variables is large and the available sample size is limited~\cite{brahim2023birds}.

Feature selection remains challenging in realistic settings characterized by observational noise, data heterogeneity, and spurious correlations. These issues are closely related to shortcut learning and dataset shift, where models tend to exploit correlations that do not transfer to deployment settings~\cite{geirhos2020shortcut,brown2023detecting,glocker2021causality}. In such settings, some variables may appear predictive because they capture spurious patterns in the training data~\cite{arjovsky2019invariant,sagawa2020groupdro,madry2018robust}. As a result, sparse models optimized only for empirical prediction may select features that fit the training distribution but are unstable under perturbations or distribution shifts. These limitations motivate feature selection methods that are not only predictive, but also stable and robust to spurious correlations.

To address these limitations, we propose Adversarial LassoNet (AdLNet), a stability-driven sparse feature selection framework that integrates local input perturbations into the hierarchical sparsity mechanism of LassoNet. The key idea is to use local adversarial sensitivity as an additional training signal for sparse feature selection. Features that are predictive mainly due to spurious correlations may be more sensitive to local input perturbations, whereas more stable predictive features tend to preserve their contribution under small worst-case input changes. By combining the clean prediction loss with a perturbation-driven stability loss, AdLNet encourages the selected feature subset to be both predictive and locally stable. Importantly, the proposed method preserves the original hierarchical sparse optimization structure of LassoNet and modifies only the training objective through an adversarial stability term.

Our main contributions are threefold:
\begin{itemize}

    \item We propose Adversarial LassoNet (AdLNet), a stability-driven extension of LassoNet that integrates input-space adversarial perturbations into hierarchical sparse feature selection. The proposed objective encourages selected features to be both predictive on clean samples and stable under local worst-case perturbations.

    \item We derive a tractable first-order approximation for local adversarial sensitivity under a local smoothness assumption and incorporate it into the hierarchical proximal optimization pipeline. We further provide an NTK-inspired spectral analysis showing that perturbation-driven training can increase the empirical NTK effective rank when a local spectral growth condition is satisfied.

    \item We conduct experiments on high-dimensional SERS data, six public benchmark datasets, and ColoredMNIST. The results show that AdLNet maintains competitive sparse-selection performance while improving out-of-distribution robustness and feature support reproducibility under spurious correlation shifts.

\end{itemize}

\section{Related Work}

\subsection{Sparse Feature Selection and Stability}
Conventional feature selection methods often rely on data-driven $\ell_1$ regularization. For example, LASSO identifies salient variables by imposing an $\ell_1$ penalty on model coefficients \citep{tibshirani1996lasso}, while subsequent extensions introduce structural constraints to model prior correlations among features \citep{meier2008group}. Building on these linear foundations, recent deep sparse models enable representation learning and feature selection to be performed jointly. LassoNet \citep{lemhadri2021lassonet}, for instance, embeds a global feature selection mechanism into neural network training through hierarchical sparsity constraints. However, these $\ell_1$-based paradigms are primarily optimized for in-distribution predictive accuracy, and the stability of the selected features may therefore degrade under distribution shifts.

To improve the stability of sparse selection, several robust optimization techniques have been explored. In linear settings, Adversarial Lasso \citep{li2021adversarial} integrates adversarial training with $\ell_1$ regularization and uses perturbation constraints to stabilize feature selection. However, its linear formulation limits its applicability to high-dimensional problems with complex nonlinear dependencies. In deep learning, methods such as Robust Sparse Deep Feature Selection \citep{robustsparse2023} reduce model sensitivity by penalizing input-gradient magnitudes. Although such penalties can improve stability to some extent, they are typically isotropic magnitude-based regularizers and may not fully exploit the directional information provided by worst-case perturbations for stable feature selection.

\subsection{Perturbation-Based Robustness}
Beyond explicit feature selection, perturbation-based learning has become an important paradigm for improving model robustness. Techniques such as adversarial training \citep{goodfellow2015adversarial}, input-gradient regularization \citep{ross2018inputgrad, madry2018robust}, and Sharpness-Aware Minimization (SAM) \citep{foret2021sam} encourage local consistency under perturbations and favor flatter, less perturbation-sensitive loss landscapes. Recently, sharpness-aware optimization has also been extended to sparse training; for example, single-step SAM variants have been used to improve the efficiency and accuracy of sparse network training \citep{ji2024ssam}. Perturbation-based robustness is particularly relevant in medical AI, where models are often expected to remain reliable under noise, distribution shifts, and spurious correlations \citep{finlayson2019adversarial, han2021advancing}. 

Nevertheless, many perturbation-based methods primarily target predictor robustness or parameter-level sparsity, such as network pruning, rather than input-level variable selection. As a result, improved robustness of the predictor does not necessarily imply that the selected feature support is concise, interpretable, or reproducible. Although our use of local perturbations is conceptually related to counterfactual explanations \citep{wachter2018counterfactual}, the goal of our method is different: we use local perturbations as a training signal to guide stable sparse feature selection, rather than to generate counterfactual samples or make causal intervention claims.

\section{Adversarial LassoNet}

\subsection{Revisiting LassoNet}
\label{sec:revist_lasso}

We consider a supervised classification problem with a training set $\mathcal{D}=\{(\mathbf{x}_i,y_i)\}_{i=1}^{N}$, where $\mathbf{x}_i\in\mathbb{R}^{d}$ denotes the input feature vector and $y_i\in\{1,\dots,C\}$ denotes the corresponding class label. The objective is to learn a predictor and a sparse subset of informative features that preserve predictive performance.

LassoNet~\cite{lemhadri2021lassonet} combines a nonlinear neural network branch with a linear skip connection from the input to the output. Specifically, the prediction function is formulated as:
\begin{equation}
    f_{\boldsymbol{\theta}}(\textbf{x}) = \boldsymbol{\theta}^\top \textbf{x} + f_\textbf{W}(\textbf{x})
\end{equation}
where $\boldsymbol{\theta} \in \mathbb{R}^{d \times C}$ denotes the weight matrix of the linear skip connection, and $f_\textbf{W}(\textbf{x})$ represents a nonlinear feedforward neural network parameterized by $\textbf{W}$. To impose structured feature sparsity, LassoNet introduces a group-sparsity regularization penalty during the joint optimization of the empirical loss $\mathcal{L}$, while imposing a hierarchical proximal constraint. The core optimization problem can be precisely formulated as the operator $\mathrm{HierProx}_{\eta\lambda,\textbf{M}}$ in LassoNet:

\begin{equation}
\begin{aligned}
\min_{\boldsymbol{\theta}, \mathbf{W}} \quad 
& \mathcal{L}(\boldsymbol{\theta}, \mathbf{W}) 
+ \lambda \sum_{j=1}^{d}\|\boldsymbol{\theta}_j\|_2 \\
\text{s.t.} \quad 
& \|\mathbf{W}_j^{(1)}\|_\infty \le M \|\boldsymbol{\theta}_j\|_2,
\quad j=1,2,\dots,d .
\end{aligned}
\label{eq:lassnet}
\end{equation}

where $\lambda$ is the regularization parameter controlling the level of sparsity, $\mathbf{W}_j^{(1)}$ denotes the weight vector connecting the $j$-th input feature to the first hidden layer of the network, and $M > 0$ is a hyperparameter that governs the relative strength of the linear and nonlinear components. The $\ell_\infty$ constraint ensures a strict hierarchical relationship: an input feature is allowed to possess non-zero weights in the nonlinear branch only if its corresponding linear representative remains active.
Consequently, LassoNet integrates feature selection directly into end-to-end network training, performing strictly structured feature sparsification while learning nonlinear representations.

LassoNet in Eq.~\eqref{eq:lassnet} performs feature selection based solely on empirical predictive performance. However, in high-dimensional yet small-sample regimes, such correlation-driven selection tends to prioritize spurious features that capture incidental correlations in the training distribution. A reasonable assumption is that some spurious features are sensitive to minor input perturbations, yielding unstable and non-reproducible sparse feature supports. To resolve this limitation, we propose a stability-driven method that encourages selected features to be both predictive and robust against worst-case local perturbations.

\subsection{Adversarial Perturbation for Stability}
\label{sec:Adversarial Perturbation}

We introduce the concept of adversarial sensitivity to explicitly measure and subsequently suppress feature instability. The adversarial sensitivity is defined for each sample-label pair $(\textbf{x},y)$ as follows:
\begin{equation}
S_{\rho}(\textbf{x},y;\theta)
:=
\max_{\|r\|_{p}\le \rho}
\left[
\mathcal{L}(f_{\theta}(\textbf{x}+\textbf{r}),y)-\mathcal{L}(f_{\theta}(\textbf{x}),y)
\right],
\label{eq:local_sensitivity}
\end{equation}
where $\rho>0$ is the perturbation radius, $\|\cdot\|_{p}$ denotes the vector $p$-norm, and $\textbf{r} \in \mathbb{R}^d$ denotes a bounded local perturbation. Eq.~\eqref{eq:local_sensitivity} follows the perturbation-based robustness perspective in adversarial training, where worst-case local input perturbations quantify model sensitivity~\cite{goodfellow2015adversarial,madry2018robust,ross2018inputgrad}.

A large value of $S_{\rho}(\textbf{x},y;\theta)$ indicates that the prediction at $\textbf{x}$ is highly sensitive to small worst-case perturbations. In our framework, we operationalize this stability by deriving the worst-case perturbation $\textbf{r}^*$, which acts as a stability-aware training signal to guide the sparsity to prune unstable features. Since solving the exact inner maximization in Eq.~\eqref{eq:local_sensitivity} can be computationally expensive, we approximate it using a local first-order expansion.

Note that we focus on the \(\ell_2\)-bounded case, namely $p=2$. Extensions to other norms follow from the corresponding dual-norm formulation.

\begin{assump}
For fixed $(\theta,y)$, we define the per-example loss 
$\ell(\textbf{x}):=\mathcal{L}(f_{\theta}(\textbf{x}),y)$. 
We assume that $\ell(\textbf{x})$ is differentiable in a neighborhood of $\textbf{x}$ and that its input gradient is locally Lipschitz. Formally, there exists a constant $L_\textbf{x}>0$ such that for all perturbations $\textbf{r}$ satisfying $\|\textbf{r}\|_2\le \rho$,
\begin{equation}
\|\nabla_\textbf{x} \ell(\textbf{x}+\textbf{r})-\nabla_\textbf{x} \ell(\textbf{x})\|_2
\le
L_\textbf{x}\|\textbf{r}\|_2.
\label{eq:lipschitz_gradient}
\end{equation}
\label{assumpion:smooth}
\end{assump}

Assumption~\ref{assumpion:smooth} is well-justified for neural networks~\cite{goodfellow2015adversarial}. Modern activation functions (e.g., ReLU, GELU, sigmoid) are Lipschitz continuous, and their layered composition yields locally Lipschitz input gradients. Under the bounded perturbation constraint \(\|\textbf{r}\|_2 \le \rho\), input $\textbf{x}$ stays within a compact neighborhood, where the loss gradient maintains uniform smoothness. 

Given Assumption~\ref{assumpion:smooth}, the local adversarial sensitivity in Eq.\eqref{eq:local_sensitivity} admits the first-order approximation as follows:
\begin{equation}
S_{\rho}(\textbf{x},y;\theta)
=
\max_{\|\textbf{r}\|_2\le \rho}
\langle \nabla_\textbf{x} \ell(\textbf{x}), \textbf{r}\rangle
+
O(\rho^2).
\label{eq:first_order_sensitivity}
\end{equation}

By the Cauchy--Schwarz inequality, the linearized inner maximization in Eq.~\eqref{eq:first_order_sensitivity} yields both the worst-case local loss increase and the perturbation direction that attains it. Specifically, to maximize the inner product, the optimal perturbation $\textbf{r}^*$ must align with the direction of the input gradient:

\begin{equation}
\begin{split}
\max_{\|\mathbf{r}\|_2\le \rho}
\langle \nabla_\mathbf{x}\ell(\mathbf{x}), \mathbf{r}\rangle
=
\rho\|\nabla_\mathbf{x}\ell(\mathbf{x})\|_2, \\
\mathbf{r}^{*}
=
\rho
\frac{\nabla_\mathbf{x}\ell(\mathbf{x})}
{\|\nabla_\mathbf{x}\ell(\mathbf{x})\|_2},
\quad
\text{if } \nabla_\mathbf{x}\ell(\mathbf{x})\neq 0 .
\end{split}
\label{eq:worst_case_perturbation}
\end{equation}
In implementation, we use the stabilized form
\[
\mathbf{r}^{*}
=
\rho
\frac{\nabla_\mathbf{x}\ell(\mathbf{x})}
{\|\nabla_\mathbf{x}\ell(\mathbf{x})\|_2+\epsilon},
\]
where $\epsilon>0$ is a small numerical constant introduced for stability. A detailed proof is provided in Appendix~\ref{app:local_sensitivity}.

The worst-case perturbation $\boldsymbol{r}^*$ provides a critical stability signal that can be seamlessly integrated into the feature selection mechanism. By embedding this perturbation into the training loop, the framework is explicitly guided to preferentially prune unstable features, encouraging the selected subset to retain features whose predictive contribution is more stable under local perturbations.

\subsection{Stability-Driven Objective and Optimization}

We build a unified stability-driven training objective and preserve the original hierarchical sparse optimization pipeline of LassoNet. We replace the standard ERM loss with the following mixed objective:

\begin{equation}
L_{\mathrm{mix}}
=
(1-\alpha)\,\underbrace{\mathcal{L}(f_{\theta}(\mathbf{x}),y)}_{\text{predictive fidelity}}
+
\alpha\,\underbrace{\mathcal{L}(f_{\theta}(\mathbf{x}+\mathbf{r}^{*}),y)}_{\text{adversarial stability}},
\label{eq:mix_loss}
\end{equation}
where $\alpha\in[0,1]$ controls the balance between clean predictive fidelity and local adversarial stability. Eq.~\eqref{eq:mix_loss} encourages the selection of features that support accurate clean prediction and remain stable under worst-case local perturbations. Rather than modifying the structural sparsification mechanism of LassoNet, our method retains the original hierarchical proximal scheme~\cite{lemhadri2021lassonet} and refines the training signal. Sparse feature selection is therefore shaped indirectly through stability-aware optimization, rather than through a separate feature-wise stability score.

\textbf{Optimization: }The resulting objective is both non-convex and non-smooth because it combines a deep neural parameterization with the hierarchical sparsity-inducing regularization inherited from LassoNet. We optimize it by alternating between adversarial perturbation construction in the input space and hierarchical proximal updates in the parameter space. Concretely, for each mini-batch, we first compute the input-gradient-based perturbation in Eq.~\eqref{eq:worst_case_perturbation}. During the outer parameter update, $\mathbf{r}^{*}$ is treated as a fixed perturbation constructed from the current mini-batch, following standard first-order adversarial training practice. We then evaluate the mixed objective~\eqref{eq:mix_loss} on clean and perturbed samples, and apply the standard gradient and hierarchical proximal updates. The overall procedure is summarized in Algorithm~\ref{alg:adv_lassonet}.

\begin{small}
\begin{algorithm}[t]
\caption{Stability-driven optimization for LassoNet}
\label{alg:adv_lassonet}
\begin{algorithmic}[1]
\Require Training data $\mathcal{D}$, perturbation radius $\rho$, mixing coefficient $\alpha$, learning rate $\eta$, path regularization parameter $\lambda$, hierarchy parameter $M$, and numerical constant $\epsilon$
\State Initialize model parameters $\theta$
\For{each sparsity level $\lambda$ along the regularization path}
    \For{each mini-batch $\{(\textbf{x}_i,y_i)\}_{i=1}^{B}$}
        \State Construct perturbations $\textbf{r}_i^*$ from Eq.~\eqref{eq:worst_case_perturbation}
             \State Form perturbed inputs $\mathbf{x}_i^{\mathrm{adv}}=\mathbf{x}_i+\mathbf{r}_i^*$
            \State Compute the perturbed loss $\mathcal{L}(f_{\theta}(\mathbf{x}_i^{\mathrm{adv}}),y_i)$
        \State Compute mixed loss $L_{\mathrm{mix}}$ from Eq.~\eqref{eq:mix_loss}
        
        \State Take a gradient step on the regularized objective
        \State Apply the hierarchical proximal operator $\mathrm{HierProx}_{\eta\lambda,M}$
    \EndFor
\EndFor
\State \Return sparse feature support and trained predictor
\end{algorithmic}
\end{algorithm}
\end{small}

\subsection{A Local Spectral Condition to Increase Effective Rank}
\label{sec:local_effrank_condition}

The stability-driven objective in Eq.~\eqref{eq:mix_loss}, optimized by Algorithm~\ref{alg:adv_lassonet}, is designed to improve the stability of selected feature supports by incorporating local adversarial perturbations into sparse learning. However, the geometric mechanism behind this stability-driven training effect remains to be characterized. To provide such an interpretation, we turn to the empirical Neural Tangent Kernel (NTK)~\cite{jacot2018ntk}, which captures the similarity of parameter gradients across training samples. Under a local first-order perturbation model, we analyze how adversarial perturbations can affect the spectral distribution of the empirical NTK. In particular, we show that the effective rank of the NTK can increase when the perturbation-induced spectral correction satisfies a positive local growth condition, corresponding to a less concentrated empirical gradient geometry.

Let \(\boldsymbol{\Theta} \in \mathbb{R}^{N \times N}\) denote a non-vanishing symmetric positive semi-definite empirical NTK. We define its effective rank as
\begin{equation}
r_{\mathrm{eff}}(\boldsymbol{\Theta}) =
\frac{\operatorname{tr}(\boldsymbol{\Theta})^2}
{\|\boldsymbol{\Theta}\|_F^2}.
\label{eq:eff_rank_def}
\end{equation}
where \(\operatorname{tr}(\boldsymbol{\Theta})\) is the sum of its eigenvalues, and \(\|\boldsymbol{\Theta}\|_F^2\) equals the sum of squared eigenvalues. The ratio \(\operatorname{tr}(\boldsymbol{\Theta})^2/\|\boldsymbol{\Theta}\|_F^2\) quantifies the uniformity of eigenvalue distribution: it reaches the true rank when eigenvalues are perfectly equal, and decreases as eigenvalue concentration grows. Therefore, the metric in Eq.~\eqref{eq:eff_rank_def} quantifies the intrinsic dimension of a matrix through the uniformity of its eigenvalue distribution and is widely adopted in the analysis of overparameterized models~\cite{bartlett2020benign}.

For perturbed inputs, we assume the empirical NTK admits a local first-order Taylor expansion around the clean kernel $\boldsymbol{\Theta}$:
\begin{equation}
\label{eq:ntk_local_taylor}
\begin{gathered}
\widetilde{\boldsymbol{\Theta}}(\rho) = \boldsymbol{\Theta} + \rho \boldsymbol{E} + \boldsymbol{R}(\rho), \\
\text{where} \quad \|\boldsymbol{R}(\rho)\|_F = o(\rho) ,
\end{gathered}
\end{equation}
with $\boldsymbol{E}$ being a symmetric first-order correction matrix, $\boldsymbol{R}(\rho)$ denoting the higher-order residual term of the Taylor expansion, and $o(\rho)$ representing the little-o notation. We assume $\widetilde{\boldsymbol{\Theta}}(\rho) \succeq 0$ holds for sufficiently small $\rho > 0$, and let $\langle \boldsymbol{A}, \boldsymbol{B} \rangle = \operatorname{tr}(\boldsymbol{A}^\top \boldsymbol{B})$ denote the Frobenius inner product throughout.

\begin{lemma}
\label{lem:eff-rank-expansion}
Under the above conditions, the effective rank of the perturbed NTK admits the following first-order Taylor expansion with respect to $\rho$:
\begin{equation}
r_{\mathrm{eff}}(\widetilde{\boldsymbol{\Theta}}(\rho))
=
r_{\mathrm{eff}}(\boldsymbol{\Theta})
+
2\,r_{\mathrm{eff}}(\boldsymbol{\Theta})
\underbrace{\left[
\frac{\operatorname{tr}(\boldsymbol{E})}{\operatorname{tr}(\boldsymbol{\Theta})}
-
\frac{\langle \boldsymbol{\Theta}, \boldsymbol{E}\rangle}{\|\boldsymbol{\Theta}\|_F^2}
\right]}_{G(\boldsymbol{E};\boldsymbol{\Theta})}\rho
+
o(\rho).
\label{eff_rank_expansion}
\end{equation}
\end{lemma}
The term $\operatorname{tr}(\boldsymbol{E})/\operatorname{tr}(\boldsymbol{\Theta})$ in~\eqref{eff_rank_expansion} measures the trace-normalized spectral mass added by the perturbation, whereas $\langle \boldsymbol{\Theta}, \boldsymbol{E}\rangle / \|\boldsymbol{\Theta}\|_F^2$ measures how strongly the correction aligns with the energy already concentrated in the clean NTK. Thus, the effective rank increases when the perturbation contributes relatively more mass to weak or underrepresented directions than to the dominant kernel directions.

Lemma~\ref{lem:eff-rank-expansion} shows that the first-order variation of the effective rank is governed by $G(\boldsymbol{E};\boldsymbol{\Theta})$, which we refer to as the spectral growth score.

\begin{proposition}[Local condition for effective-rank increase]
\label{prop:local-eff-rank-increase}
If
\[
G(\boldsymbol{E};\boldsymbol{\Theta}) > 0,
\]
then there exists $\rho_0 > 0$ such that for all $0 < \rho < \rho_0$,
\[
r_{\mathrm{eff}}(\widetilde{\boldsymbol{\Theta}}(\rho)) > r_{\mathrm{eff}}(\boldsymbol{\Theta}).
\]
\end{proposition}

\begin{proof} From Lemma~\ref{lem:eff-rank-expansion}, the change in effective rank is given by

\begin{equation}
\Delta r_{\mathrm{eff}}
=
2 r_{\mathrm{eff}}(\boldsymbol{\Theta})
G(\boldsymbol{E};\boldsymbol{\Theta})\rho
+
o(\rho).
\end{equation}
Since $r_{\mathrm{eff}}(\boldsymbol{\Theta}) > 0$ for a non-vanishing NTK, the condition $G(\boldsymbol{E};\boldsymbol{\Theta}) > 0$ implies the linear coefficient is strictly positive. By the definition of the little-o notation, the linear term strictly dominates the remainder $o(\rho)$ for a sufficiently small $\rho > 0$, guaranteeing $r_{\mathrm{eff}}(\widetilde{\boldsymbol{\Theta}}(\rho)) > r_{\mathrm{eff}}(\boldsymbol{\Theta})$.
\end{proof}

The theoretical condition $G(\boldsymbol{E};\boldsymbol{\Theta})>0$ derived above provides a local mechanism for effective-rank elevation. 
Interested readers may refer to Appendix~\ref{app:effective_rank} for details.

\section{Experiments}

\subsection{Experimental Setup}
\label{sec:exp_setup}

\textbf{Datasets.}
We conduct experiments on three groups of datasets: the Surface-Enhanced Raman Spectroscopy (SERS) lung cancer screening dataset, six public benchmark datasets, and the ColoredMNIST benchmark.

\begin{itemize}
    \item \textbf{SERS lung cancer screening dataset.}
    The SERS dataset is used to evaluate sparse feature selection in a high-dimensional medical classification scenario. It contains 87 samples, including 47 positive and 40 negative cases, and each sample consists of 1,264 spectral variables. The dataset is randomly split into training, validation, and test sets with a ratio of 70\%/10\%/20\%. This setting is consistent with recent Raman spectroscopy studies, where machine-learning models are used to classify disease-related spectra and identify diagnostically informative spectral patterns~\cite{kothari2021raman,conti2023raman}. All spectral variables are standardized using the statistics of the training set.

    \item \textbf{Public benchmark datasets.}
    We further evaluate predictive performance on six widely used public benchmark datasets, including MNIST~\cite{lecun1998gradient}, MNIST-Fashion~\cite{xiao2017fashion}, ISOLET~\cite{fanty1991spoken}, COIL-20~\cite{nene1996coil20}, Activity~\cite{anguita2013public}, and Mice Protein~\cite{henschel2015mice}. These datasets cover image recognition, speech recognition, human activity recognition, and biological classification tasks, providing diverse feature distributions and data complexities for evaluating sparse feature selection. For image datasets, pixel values are normalized to $[0,1]$; for non-image datasets, continuous features are standardized using training-set statistics.

    \item \textbf{ColoredMNIST benchmark.}
    To evaluate robustness under spurious correlation shifts, we adopt ColoredMNIST~\cite{sagawa2020distributionally}, a controlled benchmark in which color is spuriously correlated with class labels. Following standard protocols, the training environments are constructed with color-label correlations of $\{0.1,0.2,0.3\}$, and the shifted test environment is used to evaluate out-of-distribution robustness and feature support reproducibility.
\end{itemize}

\textbf{Experimental settings.}
All experiments are implemented in PyTorch 1.8.0 with CUDA 11.1 and conducted on a workstation equipped with an Intel(R) i7-7700 CPU @ 3.60GHz, 32 GB RAM, and an NVIDIA GeForce RTX 3080 GPU. All models are trained from scratch without pretraining. We use stochastic gradient descent (SGD) with momentum $0.9$, and the learning rate is set to $10^{-3}$. Batch sizes are selected according to dataset scale, ranging from $128$ to $512$. The network configurations and training epochs follow the corresponding baseline settings. Unless otherwise specified, experiments are repeated over five independent random seeds.

Vanilla LassoNet~\cite{lemhadri2021lassonet} is used as the primary baseline because it shares the same hierarchical sparsity mechanism as AdLNet but does not include adversarial stability regularization. Therefore, the comparison between AdLNet and vanilla LassoNet directly reflects the effect of the proposed stability-driven objective. On the six public benchmark datasets, we further compare with FISTA-Net~\cite{xiang2021fistanet} and Deep-Lasso~\cite{cherepanova2023performance}, which represent optimization-driven sparse learning and gradient-regularized deep feature selection, respectively.

For the six public benchmark datasets, all methods are evaluated under the matched sparse feature budget of $k=50$ following the original LassoNet regularization-path protocol~\citep{lemhadri2021lassonet}. Specifically, we vary the sparsity regularization parameter $\lambda$ and report the result at the checkpoint satisfying the target feature budget. For the SERS dataset, the main result is reported under the constraint $n_{\mathrm{selected}}<300$ to retain sufficient spectral coverage for medical classification, and an additional controlled sparse-budget analysis is conducted to examine the influence of feature-count differences. For ColoredMNIST, methods are compared under nearly matched support sparsity, so that robustness and reproducibility improvements cannot be attributed to selecting more features.

The adversarial mixing coefficient $\alpha$ controls the trade-off between predictive fidelity and perturbation-driven stability. To avoid data leakage, $\alpha$ is selected using a validation-only protocol. The candidate values are $\{0,0.2,0.4,0.6,0.8,1.0\}$ and are evaluated only on the validation set. The model checkpoint with the best validation accuracy is then evaluated on the held-out test set. Other key hyperparameters, including the hierarchy parameter $M$ and the sparsity regularization path, are selected using the same validation protocol. For AdLNet, adversarial perturbations are generated using the first-order approximation derived in Section~\ref{sec:Adversarial Perturbation}, and optimization follows Algorithm~\ref{alg:adv_lassonet}.

\textbf{Evaluation Metrics.}
We evaluate the proposed method from three complementary perspectives: predictive performance, robustness and support reproducibility under distribution shift, and optimization-geometry diagnostics.

For the six public benchmark datasets, we report classification accuracy under the matched sparse feature budget. For the SERS medical dataset, we additionally report Sensitivity (Sens), Specificity (Spec), and area under the ROC curve (AUC), which provide a more comprehensive evaluation of binary medical classification performance.

For ColoredMNIST, we report in-distribution (ID) accuracy, out-of-distribution (OOD) accuracy, and the ID--OOD generalization gap to assess robustness under spurious correlation shifts. Feature support reproducibility is quantified using the Jaccard index across independent runs, measuring the consistency of selected feature subsets under different random initializations.

For training-dynamics analysis, we compute empirical NTK effective rank, dominant eigenvalue ratio, spectral energy concentration, and Hessian curvature statistics. These metrics are used as diagnostic evidence to examine whether stability-driven training reduces spectral concentration and sharp curvature in sparse learning. Detailed metric definitions and calculation procedures are provided in Appendix~\ref{app:metrics}.

\subsection{Predictive Performance under Sparse Feature Selection}

We first evaluate whether the proposed stability-driven objective can preserve predictive performance under sparse feature selection. This evaluation is conducted on the high-dimensional SERS medical dataset and six public benchmark datasets. The SERS experiment examines the behavior of AdLNet in a small-sample, high-dimensional regime, while the public benchmark experiments further test whether the method remains effective across diverse data domains under matched feature budgets.

Table~\ref{tab:sers_main} reports the results on the SERS lung cancer screening dataset. Compared with vanilla LassoNet, AdLNet achieves better test performance under validation-selected sparse checkpoints, particularly in terms of test accuracy and AUC. This result suggests that introducing adversarial stability into the sparse learning objective can improve out-of-sample generalization in the high-dimensional medical setting, where prediction-only sparse fitting may be more sensitive to unstable correlations in limited training data.

To further examine whether this improvement simply comes from selecting more features, we conduct a controlled sparse-budget analysis in Table~\ref{tab:sers_controlled_budget}. Under a high-sparsity regime ($k \approx 120$), AdLNet slightly underperforms vanilla LassoNet, indicating that excessive stability pressure may remove some predictive but perturbation-sensitive features and lead to underfitting. In contrast, under a moderate-sparsity regime ($k \approx 250$), AdLNet achieves higher test accuracy than vanilla LassoNet. These results reveal an important trade-off between sparsity and stability: adversarial perturbation is most beneficial when the feature budget is sufficient to retain stable predictive features, while overly restrictive sparsity may limit its advantage.

Table~\ref{tab:benchmark_main} further reports the results on six public benchmark datasets under the matched feature budget of $k=50$. AdLNet achieves the best test accuracy on MNIST, MNIST-Fashion, COIL-20, and Activity, and remains competitive on ISOLET and Mice Protein. Therefore, the proposed objective does not sacrifice sparse-selection performance on standard benchmark datasets. Together, the SERS and public benchmark results demonstrate that AdLNet maintains competitive predictive performance under sparse feature selection, while offering improved generalization in the high-dimensional medical setting.

\begin{table}[htbp]
\centering
\caption{Feature selection performance on the SERS medical dataset. The reported checkpoints are selected by validation accuracy under the constraint $n_{\mathrm{selected}}<300$.}
\label{tab:sers_main}
\small
\begin{tabular}{lcc}
\toprule
Metric & LassoNet & AdLNet \\
\midrule
Selected Features & 119 & 256 \\
Validation Accuracy (\%) & \textbf{75.7} & 73.6 \\
Test Accuracy (\%) & 72.1 & \textbf{76.4} \\
Sensitivity (\%) & \textbf{85.7} & 71.4 \\
Specificity (\%) & 77.8 & \textbf{82.5} \\
AUC & 0.775 & \textbf{0.835} \\
\bottomrule
\end{tabular}
\end{table}

\begin{table}[htbp]
\centering
\caption{Test accuracy (\%) on public benchmark datasets under the matched sparse-selection protocol. Results are averaged over five random seeds with target feature budget $k=50$.}
\label{tab:benchmark_main}
\resizebox{\columnwidth}{!}{
\begin{tabular}{lcccc}
\toprule
Dataset & Vanilla LassoNet & FISTA-Net & Deep-Lasso & AdLNet \\
\midrule
Mice Protein           & 90.2 $\pm$ 2.2 & \textbf{94.2 $\pm$ 2.7} & 91.2 $\pm$ 1.6 & 91.9 $\pm$ 2.2 \\
MNIST          & 88.5 $\pm$ 0.9 & 89.0 $\pm$ 11.3 & 86.0 $\pm$ 1.6 & \textbf{89.4 $\pm$ 0.9} \\
MNIST-Fashion  & 77.3 $\pm$ 1.0 & 76.3 $\pm$ 5.7 & 78.9 $\pm$ 0.6 & \textbf{79.0 $\pm$ 1.7} \\
ISOLET         & 82.0 $\pm$ 1.1 & \textbf{85.4 $\pm$ 2.9} & 82.7 $\pm$ 1.6 & 83.9 $\pm$ 1.0 \\
COIL-20        & 98.5 $\pm$ 0.9 & 96.5 $\pm$ 2.7 & 98.3 $\pm$ 0.8 & \textbf{99.2 $\pm$ 0.6} \\
Activity       & 92.9 $\pm$ 0.4 & 92.6 $\pm$ 0.8 & 93.0 $\pm$ 0.2 & \textbf{94.5 $\pm$ 0.4} \\
\bottomrule
\end{tabular}
}
\end{table}

\subsection{Robustness and Support Reproducibility under Spurious Correlation Shifts}

We next evaluate whether the proposed adversarial stability objective improves robustness and feature support reproducibility under spurious correlation shifts. To this end, we conduct experiments on ColoredMNIST, a controlled benchmark in which color is spuriously correlated with class labels during training but becomes unreliable in the shifted test environment. This setting directly tests whether a sparse feature-selection method relies on shortcut-sensitive features or selects features that remain predictive under distribution shift.

Table~\ref{tab:coloredmnist_main} reports the results. Compared with vanilla LassoNet, AdLNet achieves better out-of-distribution (OOD) accuracy while maintaining nearly identical support sparsity ($0.969$ versus $0.966$). This indicates that the robustness gain is not caused by selecting a denser feature subset. Instead, this improvement suggests that the stability-driven training signal introduced by adversarial perturbation contributes to better OOD robustness. Vanilla LassoNet performs sparse selection mainly according to prediction loss on the training distribution, and therefore may retain shortcut-sensitive features that exploit the color-label correlation. In contrast, AdLNet incorporates local worst-case perturbations into the hierarchical sparsity framework, encouraging the proximal updates to preserve features whose predictive contribution is more stable under local input perturbations.

In addition to OOD robustness, AdLNet also improves feature support reproducibility, as reflected by the higher Jaccard index across independent runs. This suggests that adversarial stability regularization makes the selected feature subsets less sensitive to random initialization and spurious training correlations. Overall, the ColoredMNIST results support our core hypothesis that integrating adversarial perturbation into sparse feature selection can improve both robustness and support reproducibility under spurious correlation shifts.

\begin{table}[htbp]
\centering
\caption{ColoredMNIST results under nearly matched support sparsity. ID and OOD accuracies are evaluated against the true labels. Gap denotes the ID--OOD accuracy difference, and Jaccard measures support reproducibility across independent runs.}
\label{tab:coloredmnist_main}
\resizebox{\columnwidth}{!}{
\begin{tabular}{lccccc}
\toprule
Method & ID Acc (\%) & OOD Acc (\%) & Gap (\%) & Support Sparsity & Jaccard \\
\midrule
LassoNet & 84.6 $\pm$ 0.6 & 47.3 $\pm$ 4.9 & 37.3 $\pm$ 4.3 & 0.969 & 0.456 \\
AdLNet   & \textbf{85.5 $\pm$ 1.0} & \textbf{51.7 $\pm$ 2.7} & \textbf{33.8 $\pm$ 1.7} & 0.966 & \textbf{0.519} \\
\bottomrule
\end{tabular}
}
\end{table}

\subsection{Spectral Diagnostics of Optimization Geometry}

We further analyze the empirical NTK and Hessian spectra to examine whether stability-driven training changes the optimization geometry of sparse learning. The local analysis in Section~\ref{sec:local_effrank_condition} suggests that adversarial perturbation can increase the effective rank of the empirical NTK when the first-order spectral growth condition holds. Since this condition is derived for local infinitesimal perturbations, while our experiments compare trained models, the following analysis should be interpreted as an empirical diagnostic rather than a direct proof of the local condition.

To quantify the macroscopic redistribution of NTK spectral mass, we compute the following empirical spectral-redistribution score:
\[
\hat{G}_{\mathrm{macro}}
=
\frac{\mathrm{tr}(\boldsymbol{\Theta}_{\rho}-\boldsymbol{\Theta}_{0})}{\mathrm{tr}(\boldsymbol{\Theta}_{0})}
-
\frac{\langle \boldsymbol{\Theta}_{0},\boldsymbol{\Theta}_{\rho}-\boldsymbol{\Theta}_{0}\rangle}{\|\boldsymbol{\Theta}_{0}\|_F^2},
\]
where
\[
\boldsymbol{\Theta}_{0}=J(\mathbf{X};\boldsymbol{\theta}_{0})J(\mathbf{X};\boldsymbol{\theta}_{0})^\top,
\qquad
\boldsymbol{\Theta}_{\rho}=J(\mathbf{X};\boldsymbol{\theta}_{\rho})J(\mathbf{X};\boldsymbol{\theta}_{\rho})^\top.
\]
Here, $\boldsymbol{\theta}_{0}$ denotes the vanilla LassoNet checkpoint, and $\boldsymbol{\theta}_{\rho}$ denotes the AdLNet checkpoint. Both empirical NTKs are evaluated on the same clean inputs $\mathbf{X}$, so that the comparison reflects differences induced by the learned parameter configurations rather than by different input samples. A positive value of $\hat{G}_{\mathrm{macro}}$ suggests that stability-driven training redistributes spectral mass toward less dominant NTK directions at the macroscopic level.

In our experiment, we obtain $\hat{G}_{\mathrm{macro}}=0.121$, indicating a positive empirical spectral-redistribution trend. As shown in Table~\ref{tab:spectral_summary}, AdLNet produces a less concentrated NTK spectrum than vanilla LassoNet, with a higher effective rank and lower dominant-eigenvalue energy concentration. This suggests that adversarial stability regularization suppresses excessively dominant optimization directions and encourages a more distributed gradient geometry under sparse learning.

A similar trend is observed in the Hessian spectrum. Compared with vanilla LassoNet, AdLNet yields a smoother local curvature profile with weaker sharp directions, indicating improved local optimization stability under perturbation-aware sparse learning. Together, the NTK and Hessian diagnostics are consistent with the spectral interpretation in Section~\ref{sec:local_effrank_condition}: stability-driven training tends to reduce spectral concentration and increase the effective rank of the empirical NTK. However, these results should be viewed as finite-sample diagnostic evidence rather than a strict verification of the local condition $G(\boldsymbol{E};\boldsymbol{\Theta})>0$.

\begin{table}[t]
\centering
\caption{NTK and Hessian spectral diagnostics on ColoredMNIST. The NTK metrics characterize spectral concentration and effective rank, while the Hessian metrics characterize the local curvature profile.}
\label{tab:spectral_summary}
\resizebox{\columnwidth}{!}{
\begin{tabular}{llcc}
\toprule
Spectrum & Metric & LassoNet & AdLNet \\
\midrule
\multirow{5}{*}{NTK}
& Trace $\mathrm{Tr}(\Theta)$                              & 15018.42 & 10219.46 \\
& Effective Rank $\uparrow$                                & 7.07     & \textbf{8.05} \\
& $\lambda_1 / \mathrm{Tr}(\Theta) \downarrow$             & 0.343    & \textbf{0.310} \\
& Top-10 Energy Ratio $\downarrow$                         & 0.740    & \textbf{0.722} \\
& Spectral Growth Coefficient $\hat{G}_{\mathrm{macro}} \uparrow$ & -- & \textbf{0.121} \\
\midrule
\multirow{3}{*}{Hessian}
& Maximum Eigenvalue $\lambda_{\max} \downarrow$           & 14.08    & \textbf{10.13} \\
& Trace $\mathrm{Tr}(H) \downarrow$                        & 50.10    & \textbf{42.66} \\
& Average Curvature $\bar{\kappa} \downarrow$              & 0.805    & \textbf{0.726} \\
\bottomrule
\end{tabular}
}
\end{table}

\subsection{Component-wise Ablation Study}

We conduct ablation experiments on the Mice Protein dataset to examine the contribution of each component in AdLNet. Specifically, we compare the full model with vanilla LassoNet and several partial variants, including LassoNet with only the stability-driven loss, LassoNet with adversarial perturbation, and LassoNet with a gradient-norm penalty replacing directional adversarial perturbation.

As shown in Table~\ref{tab:ablation}, the full AdLNet model achieves the best overall performance, with the highest test accuracy ($92.0\%$) and the highest feature support consistency measured by the Jaccard index ($0.71$). Compared with the partial variants that use only the stability-driven loss or only adversarial perturbation, the full model achieves better predictive performance and support reproducibility. This indicates that the observed improvement cannot be explained by either component alone, but arises from their joint integration within the hierarchical sparsity framework.

Replacing the directional adversarial perturbation with a gradient-norm penalty also degrades performance and introduces additional computational cost. This result suggests that the key benefit of AdLNet comes from explicitly probing locally vulnerable input directions, rather than simply imposing a generic smoothness regularization. In other words, the directional perturbation provides a more informative stability signal for guiding sparse feature selection.

Overall, the ablation results demonstrate that robust prediction and reproducible feature selection rely on the joint effect of perturbation-based stability regularization and hierarchical sparse support formation.

\begin{table}[t]
\centering
\caption{Ablation study on Mice Protein. The full model achieves the best test accuracy and support consistency, while the gradient-norm penalty baseline incurs the largest computational cost.}
\label{tab:ablation}
\resizebox{\columnwidth}{!}{
\begin{tabular}{lccccc}
\toprule
Method & Val Acc (\%) & Test Acc (\%) & Jaccard & Selected Ratio (\%) & Time (s)\\
\midrule
LassoNet                    & 89.8 $\pm$ 2.5 & 90.6 $\pm$ 3.7 & 0.66 & 34.8 & $215.3 \pm 1.9$ \\
LassoNet + Stability Loss   & 91.3 $\pm$ 2.7 & 90.4 $\pm$ 2.7 & 0.69 & 34.5 & $305.6 \pm 11.5$ \\
LassoNet + Adversarial      & 90.6 $\pm$ 2.0 & 91.5 $\pm$ 3.6 & 0.67 & 34.8 & $311.8 \pm 20.0$ \\
LassoNet + Gradient Norm    & 90.1 $\pm$ 3.1 & 90.6 $\pm$ 3.6 & 0.65 & 33.8 & $552.8 \pm 60.6$ \\
Full model                  & \textbf{91.5 $\pm$ 1.8} & \textbf{92.0 $\pm$ 1.2} & \textbf{0.71} & 37.6 & $472.4 \pm 5.6$ \\
\bottomrule
\end{tabular}
}
\end{table}

\subsection{Sensitivity to the Mixing Coefficient}
\label{sec:alpha_sensitivity}

We finally analyze the effect of the adversarial mixing coefficient $\alpha$, which controls the balance between predictive fidelity and perturbation-driven stability. As shown in Table~\ref{tab:alpha_sensitivity_full}, AdLNet maintains relatively stable performance across a broad range of $\alpha$ values on most datasets, indicating that the proposed framework is not overly sensitive to a narrowly tuned mixing coefficient.

The results also reflect the trade-off introduced by the mixed objective. A smaller $\alpha$ places more emphasis on the original prediction objective, while a larger $\alpha$ increases the influence of perturbation-driven stability. Therefore, moderate values of $\alpha$ can better balance clean predictive information and stability regularization in many cases. The optimal value of $\alpha$ varies across datasets, suggesting that this balance is data-dependent.

\begin{table}[htbp]
\centering
\caption{Sensitivity analysis of the adversarial mixing coefficient $\alpha$. Validation and test accuracies are reported under different $\alpha$ values. Bold values indicate the best validation accuracy for each dataset.}
\label{tab:alpha_sensitivity_full}
\resizebox{\columnwidth}{!}{%
\begin{tabular}{lcccccccccccc}
\toprule
Value of $\alpha$ & \multicolumn{2}{c}{0} & \multicolumn{2}{c}{0.2} & \multicolumn{2}{c}{0.4} & \multicolumn{2}{c}{0.6} & \multicolumn{2}{c}{0.8} & \multicolumn{2}{c}{1.0} \\
\cmidrule(lr){2-3} \cmidrule(lr){4-5} \cmidrule(lr){6-7} \cmidrule(lr){8-9} \cmidrule(lr){10-11} \cmidrule(lr){12-13}
Dataset & Val & Test & Val & Test & Val & Test & Val & Test & Val & Test & Val & Test \\
\midrule
Mice Protein          & 91.3 & 90.4 & 92.0 & 91.6 & 90.2 & 90.6 & 90.6 & 91.8 & 91.3 & 90.8 & \textbf{92.2} & 90.9 \\
MNIST         & 89.1 & 89.0 & \textbf{89.4} & 89.2 & 88.5 & 88.4 & 88.8 & 88.0 & 88.9 & 88.6 & 87.8 & 87.7 \\
MNIST-Fashion & 78.1 & 78.2 & \textbf{79.3} & 79.1 & 78.8 & 79.4 & 79.0 & 79.6 & 79.0 & 78.7 & 78.5 & 79.1 \\
ISOLET        & 83.2 & 83.2 & 84.4 & 84.2 & 83.4 & 83.5 & \textbf{84.6} & 83.7 & 82.7 & 82.5 & 82.6 & 82.1 \\
COIL-20       & 99.0 & 99.4 & \textbf{99.6} & 99.2 & 98.8 & 99.7 & 98.5 & 99.4 & 98.8 & 99.2 & 98.9 & 99.2 \\
Activity      & 98.1 & 92.2 & 98.0 & 93.3 & 98.2 & 93.4 & \textbf{98.4} & 93.5 & 98.2 & 93.5 & 98.1 & 93.5 \\
\bottomrule
\end{tabular}
}
\end{table}

\section{Conclusion}
\label{sec:conclusion}

This paper presents AdLNet, a stability-driven sparse feature selection framework that integrates local adversarial perturbations into the hierarchical sparsity mechanism of LassoNet. By combining predictive relevance with local input stability, AdLNet encourages selected features to remain informative under both clean inputs and local worst-case perturbations, while preserving the original sparse optimization structure. We derive a first-order approximation for the adversarial perturbation term and provide an NTK-inspired spectral analysis to characterize the effect of stability-driven training on sparse learning. Empirical results on high-dimensional SERS data, six public benchmark datasets, and ColoredMNIST show that AdLNet maintains competitive sparse-selection performance while improving robustness under spurious correlation shifts and enhancing feature support reproducibility. NTK and Hessian diagnostics further provide finite-sample evidence consistent with the proposed spectral interpretation, suggesting that the observed robustness improvements are associated with reduced spectral concentration and a milder local curvature profile.

The current study focuses on feature-level sparse selection within the LassoNet framework and analyzes stability mainly from the perspective of local input perturbations. This setting allows us to clearly examine the interaction between hierarchical sparsity and adversarial stability, but broader forms of sparse learning remain to be explored. Future work will investigate whether the same stability-driven principle can be extended to other sparsity structures, such as weight sparsity and attention sparsity, as well as to more diverse distribution-shift scenarios.

\appendix
\section*{Appendix}
\addcontentsline{toc}{section}{Appendix}
\section{First-Order Approximation of Local Adversarial Sensitivity}
\label{app:local_sensitivity}
This appendix provides a rigorous error bound for the first-order approximation used to construct local adversarial perturbations in the main text. For fixed parameters and label $(\theta, y)$, define the loss function as
\[
\ell(\textbf{x}) := \mathcal{L}(f_{\theta}(\textbf{x}), y).
\]
Under Assumption 1 in the main text, $\ell(\textbf{x})$ is differentiable in a neighborhood of $\textbf{x}$, and the input gradient is locally Lipschitz continuous. That is, there exists $L_\textbf{x} > 0$ such that for all perturbations $r$ with $\norm{r} \le \rho$,
\[
\norm{\grad_\textbf{x} \ell(\textbf{x}+\textbf{r}) - \grad_\textbf{x} \ell(\textbf{x})} \le L_\textbf{x} \norm{\textbf{r}}.
\]

\subsection{First-Order Expansion and Residual Bound}
By the fundamental theorem of calculus, the change in loss can be written exactly as an integral of the gradient:
\[
\ell(\boldsymbol{x}+\boldsymbol{r}) - \ell(\boldsymbol{x}) 
= \int_0^1 \bigl\langle \nabla_{\boldsymbol{x}} \ell(\boldsymbol{x}+t\boldsymbol{r}), \boldsymbol{r} \bigr\rangle \mathrm{d}t
\]
Inserting $\nabla_{\boldsymbol{x}} \ell(\boldsymbol{x})$ inside the integral, we decompose the difference into a first-order linear term and a remainder $R_2(\boldsymbol{x},\boldsymbol{r})$:
\[
\begin{aligned}
\ell(\boldsymbol{x}+\boldsymbol{r}) - \ell(\boldsymbol{x})
&= \bigl\langle \nabla_{\boldsymbol{x}} \ell(\boldsymbol{x}), \boldsymbol{r} \bigr\rangle \\
&\quad + \underbrace{\int_0^1 \bigl\langle \nabla_{\boldsymbol{x}} \ell(\boldsymbol{x}+t\boldsymbol{r}) - \nabla_{\boldsymbol{x}} \ell(\boldsymbol{x}), \boldsymbol{r} \bigr\rangle \mathrm{d}t}_{R_2(\boldsymbol{x},\boldsymbol{r})}.
\end{aligned}
\]
Taking absolute value and applying the Cauchy–Schwarz inequality together with the Lipschitz condition:
\begin{align*}
|R_2(\boldsymbol{x},\boldsymbol{r})|
&\le \int_0^1 \norm{\nabla_{\boldsymbol{x}} \ell(\boldsymbol{x}+t\boldsymbol{r}) - \nabla_{\boldsymbol{x}} \ell(\boldsymbol{x})} \norm{\boldsymbol{r}} \mathrm{d}t \\
&\le \int_0^1 L_{\boldsymbol{x}} t \norm{\boldsymbol{r}}^2 \mathrm{d}t
= \frac{L_{\boldsymbol{x}}}{2} \norm{\boldsymbol{r}}^2.
\end{align*}
Since $\|\boldsymbol{r}\| \le \rho$, the residual is bounded by $\frac{L_{\boldsymbol{x}}}{2}\rho^2$.

\subsection{Local Adversarial Sensitivity}
The local adversarial sensitivity is defined as the maximal loss increase within the $\rho$-ball:
\begin{align*}
S_{\rho}(\boldsymbol{x},y;\theta)
:= \max_{\norm{\boldsymbol{r}}\le\rho} [\ell(\boldsymbol{x}+\boldsymbol{r})-\ell(\boldsymbol{x})]\\
= \max_{\norm{\boldsymbol{r}}\le\rho} \big[\langle \grad_x \ell(\boldsymbol{x}), \boldsymbol{r} \rangle + R_2(\boldsymbol{x},\boldsymbol{r})\big].
\end{align*}
By Cauchy–Schwarz, the first-order term is maximized when $\boldsymbol{r}$ is aligned with $\nabla_{\boldsymbol{x}} \ell(\boldsymbol{x})$, giving $\rho\norm{\grad_\textbf{x} \ell(\textbf{x})}$. Combined with the residual bound, the maximum absolute error is $O(\rho^2)$. Hence the first-order approximation
\[
S_{\rho}(\textbf{x},y;\theta)
= \rho \norm{\grad_\textbf{x} \ell(\textbf{x})} + O(\rho^2)
\]
holds rigorously.

\section{First-Order Spectral Analysis of Effective Rank}
\label{app:effective_rank}
We derive the first-order Taylor expansion of the effective rank of the perturbed empirical NTK. Let $\boldsymbol{\Theta}\in\mathbb{R}^{N\times N}$ be a nonzero symmetric positive-semidefinite matrix. The effective rank is
\[
r_{\mathrm{eff}}(\boldsymbol{\Theta})
= \frac{\tr(\boldsymbol{\Theta})^2}{\fro{\boldsymbol{\Theta}}^2}.
\]
Consider the perturbed empirical NTK
\[
\widetilde{\boldsymbol{\Theta}}(\rho) = \boldsymbol{\Theta} + \rho \boldsymbol{E} + \boldsymbol{R}(\rho),
\qquad
\fro{\boldsymbol{R}(\rho)} = o(\rho),
\]
where $\boldsymbol{E}$ is symmetric and $\widetilde{\boldsymbol{\Theta}}(\rho)\succeq 0$ for sufficiently small $\rho>0$. The Frobenius inner product is $\langle \boldsymbol{A},\boldsymbol{B}\rangle = \tr(\boldsymbol{A}^\top \boldsymbol{B})$.

\begin{lemma}
Let $T=\operatorname{tr}(\boldsymbol{\Theta})>0$ and $S=\|\boldsymbol{\Theta}\|_{\mathrm{F}}^2>0$. The perturbed effective rank satisfies
\[
r_{\mathrm{eff}}(\widetilde{\boldsymbol{\Theta}}(\rho))
= r_{\mathrm{eff}}(\boldsymbol{\Theta})
+ 2\,r_{\mathrm{eff}}(\boldsymbol{\Theta})\left(
\frac{\operatorname{tr}(\boldsymbol{E})}{\operatorname{tr}(\boldsymbol{\Theta})}
- \frac{\langle\boldsymbol{\Theta},\boldsymbol{E}\rangle}{\|\boldsymbol{\Theta}\|_{\mathrm{F}}^2}
\right)\rho + o(\rho).
\]
\end{lemma}

\begin{proof}
Using linearity of the trace,
\[
\operatorname{tr}(\widetilde{\boldsymbol{\Theta}}(\rho))
= T + \rho\operatorname{tr}(\boldsymbol{E}) + o(\rho).
\]
For the denominator,
\[
\|\widetilde{\boldsymbol{\Theta}}(\rho)\|_{\mathrm{F}}^2
= \|\boldsymbol{\Theta} + \rho \boldsymbol{E} + \boldsymbol{R}(\rho)\|_{\mathrm{F}}^2
= S + 2\rho\langle\boldsymbol{\Theta},\boldsymbol{E}\rangle + o(\rho),
\]
where cross terms involving $\boldsymbol{R}(\rho)$ are $o(\rho)$. Substituting into the effective rank:
\[
r_{\mathrm{eff}}(\widetilde{\boldsymbol{\Theta}}(\rho))
= \frac{\bigl(T+\rho\operatorname{tr}(\boldsymbol{E})+o(\rho)\bigr)^2}
{S+2\rho\langle\boldsymbol{\Theta},\boldsymbol{E}\rangle+o(\rho)}
= \frac{T^2 + 2\rho T\operatorname{tr}(\boldsymbol{E}) + o(\rho)}
{S + 2\rho\langle\boldsymbol{\Theta},\boldsymbol{E}\rangle + o(\rho)}.
\]
Apply the first-order expansion
\[
\frac{A+\rho a}{B+\rho b}
= \frac{A}{B}\left(1+\rho\frac{a}{A}-\rho\frac{b}{B}\right)+o(\rho)
\]
with $A=T^2$, $a=2T\operatorname{tr}(\boldsymbol{E})$, $B=S$, $b=2\langle\boldsymbol{\Theta},\boldsymbol{E}\rangle$. Recalling $r_{\mathrm{eff}}(\boldsymbol{\Theta})=T^2/S$,
\[
r_{\mathrm{eff}}(\widetilde{\boldsymbol{\Theta}}(\rho))
= r_{\mathrm{eff}}(\boldsymbol{\Theta})
\left(1+\rho\frac{2T\operatorname{tr}(\boldsymbol{E})}{T^2}-\rho\frac{2\langle\boldsymbol{\Theta},\boldsymbol{E}\rangle}{S}\right)+o(\rho).
\]
Simplifying yields the statement of the lemma.
\end{proof}

\section{Evaluation Metrics}
\label{app:metrics}

This appendix defines the metrics used in the experiments and states the preferred direction for each metric.

\paragraph{Selected Feature Count and Sparsity.}
For a selected feature set $A$ and input dimension $d$,
\[
\text{Sparsity}=\frac{|A|}{d}.
\]
Lower sparsity indicates a smaller selected support, provided that predictive performance is maintained.

\paragraph{Jaccard Index.}
For two selected feature sets $A$ and $B$ from independent runs,
\[
J(A,B)=\frac{|A\cap B|}{|A\cup B|}.
\]
Higher values indicate more reproducible feature selection.

\paragraph{ID and OOD Accuracy.}
Classification accuracy is
\[
\text{Accuracy}=\frac{\text{Number of correct predictions}}{\text{Total number of samples}}.
\]
Higher values indicate better predictive performance. ID accuracy is measured on in-distribution data, whereas OOD accuracy is measured under the specified distribution shift.

\paragraph{ID--OOD Gap.}
We quantify robustness under shift as
\[
\text{Gap}=\text{ID Accuracy}-\text{OOD Accuracy}.
\]
Lower values indicate a smaller degradation from ID to OOD evaluation.

\paragraph{Empirical NTK Effective Rank.}
For an empirical NTK matrix $\Theta\in\mathbb{R}^{N\times N}$,
\[
r_{\mathrm{eff}}(\Theta)=\frac{\operatorname{tr}(\Theta)^2}{\|\Theta\|_F^2}.
\]
Higher values indicate a less concentrated NTK spectrum.

\paragraph{Leading Eigenvalue Ratio.}
Let $\lambda_1(\Theta)$ be the largest eigenvalue of $\Theta$. We report
\[
\frac{\lambda_1(\Theta)}{\operatorname{tr}(\Theta)}.
\]
Lower values indicate that the dominant spectral mode accounts for a smaller fraction of total kernel energy.

\paragraph{Top-$k$ Energy Ratio.}
For eigenvalues $\lambda_1(\Theta)\ge \cdots \ge \lambda_N(\Theta)\ge 0$,
\[
\frac{\sum_{i=1}^{k}\lambda_i(\Theta)}{\operatorname{tr}(\Theta)}
\]
measures the fraction of spectral energy in the leading $k$ directions. We use $k=10$. Lower values indicate less concentration in the leading eigendirections.

\paragraph{Hessian Diagnostics.}
Let $H$ be the Hessian of the training objective with respect to model parameters. We report the maximum eigenvalue $\lambda_{\max}(H)$, the trace $\operatorname{tr}(H)$, and the average curvature
\[
\bar{\kappa}=\frac{\operatorname{tr}(H)}{\dim(H)}.
\]
Lower values indicate milder local curvature.

\section{Supplementary Analysis on Controlled Feature Budgets}
\label{app:controlled_budget}

To complement the SERS evaluation in the main text, we report controlled feature-budget results under two sparsity regimes. This analysis separates the effect of adversarial stability from the effect of selecting more features.

\begin{table}[htbp]
\centering
\small
\caption{Controlled feature-budget analysis on the SERS dataset. ``Standard'' denotes vanilla LassoNet, and ``Adversarial'' denotes the proposed stability-driven framework. Absolute changes ($\Delta$) compare the adversarial model with the standard model under the same sparsity regime.}
\label{tab:sers_controlled_budget}
\resizebox{\columnwidth}{!}{
\begin{tabular}{lcccccc}
\toprule
\textbf{Budget / Model} & \textbf{Val Acc} & \textbf{Test Acc} & \textbf{Sens} & \textbf{Spec} & \textbf{AUC} & \textbf{Features} \\
\midrule
\multicolumn{7}{l}{\textit{\textbf{High-Sparsity Regime ($k \approx 120$)}}} \\
Standard (clean)  & 75.71 & 72.14 & 85.71 & 58.57 & 0.775 & 119 \\
Adversarial (sam) & 47.86 & 52.86 & 62.86 & 42.86 & 0.482 & 122 \\
$\Delta$ (sam $-$ clean) & $-27.86$ & $-19.29$ & $-22.86$ & $-15.71$ & $-0.293$ & $+3$ \\
\midrule
\multicolumn{7}{l}{\textit{\textbf{Moderate-Sparsity Regime ($k \approx 250$)}}} \\
Standard (clean)  & 57.14 & 62.14 & 85.71 & 38.57 & 0.672 & 272 \\
Adversarial (sam) & 73.57 & 76.43 & 71.43 & 81.43 & 0.825 & 256 \\
$\Delta$ (sam $-$ clean) & $+16.43$ & $+14.29$ & $-14.29$ & $+42.86$ & $+0.154$ & $-16$ \\
\bottomrule
\end{tabular}
}
\end{table}

Under the high-sparsity budget, the adversarial model selects nearly the same number of features as the standard model but performs worse. Under the moderate-sparsity budget, it improves test accuracy and AUC while selecting fewer features. These results suggest that the SERS gains in the main experiments are not explained by feature-count expansion alone. Instead, the effect of adversarial stability depends on whether the feature budget allows a sufficiently informative and stable support to be retained.

\bibliographystyle{named}
\bibliography{ijcai25}
\end{document}